\documentclass[11pt]{article}

\usepackage[margin=1.1in]{geometry}
\usepackage[T1]{fontenc}
\usepackage[utf8]{inputenc}
\usepackage{lmodern}
\usepackage{microtype}
\usepackage{amsmath,amssymb,amsthm,mathtools}
\usepackage{graphicx}
\usepackage{booktabs,tabularx,needspace,placeins}
\usepackage[authoryear,round]{natbib}
\usepackage[colorlinks=true,linkcolor=blue,citecolor=blue,urlcolor=blue]{hyperref}

\numberwithin{equation}{section}

\theoremstyle{plain}
\newtheorem{theorem}{Theorem}[section]
\newtheorem{corollary}[theorem]{Corollary}
\newtheorem{lemma}[theorem]{Lemma}
\newtheorem{proposition}[theorem]{Proposition}
\theoremstyle{definition}
\newtheorem*{definition}{Definition}
\theoremstyle{remark}
\newtheorem*{remark}{Remark}
\theoremstyle{plain}

\DeclareMathOperator{\Var}{Var}
\DeclareMathOperator{\Cov}{Cov}

\DeclareMathOperator{\TC}{TC}

\newcommand{\R}{\mathbb{R}}

\newcommand{\bP}{\mathbb{P}}
\newcommand{\bE}{\mathbb{E}}

\newcommand{\norm}[1]{\left\lVert #1 \right\rVert}
\newcommand{\abs}[1]{\left\lvert #1 \right\rvert}

\newcommand{\rhoV}{\rho_{\mathrm V}}
\newcommand{\rhoX}{\rho_{\mathrm X}}

\title{One-step lowest-variance selection in a Gaussian random-field model
motivated by masked diffusion:\\
Total correlation and a square-root collision threshold}
\author{Linjun Li
\thanks{Department of Mathematics, University of Pennsylvania, Philadelphia, PA.} 
}
\date{}

\hypersetup{
  pdftitle={One-step lowest-variance selection in a Gaussian random-field model motivated by masked diffusion: Total correlation and a square-root collision threshold},
  pdfauthor={Linjun Li},
  pdfsubject={A stylized one-step stochastic-geometry model motivated by confidence-based masked diffusion schedules},
  pdfkeywords={masked diffusion motivation, one-step selection, Gaussian random field, stochastic geometry, total correlation, Poisson approximation}
}

\begin{document}
\maketitle
\begin{abstract}
Motivated by confidence-guided parallel unmasking in masked discrete diffusion, we study a single selection step in a stylized Gaussian random-field model. A locally dependent nonnegative score field represents positionwise uncertainty, and the scheduler selects the K positions with the smallest scores. Dependence among the selected positions is measured through a distance-dependent Gaussian correlation model. This separation provides a tractable framework for quantifying how the geometry of low-score locations affects the dependence cost of factorized parallel decoding. We establish two complementary results. In a conservative sub-square-root regime, the conditional Gaussian total correlation of the selected block vanishes in probability. At the square-root scale, it remains non-negligible with positive asymptotic probability and admits a strictly positive expectation lower bound. Synthetic experiments support the predicted finite-size behavior. These results provide a rigorous stochastic-geometry baseline for understanding how budget size, score dependence, and spatial correlation jointly shape one-step confidence-based selection in masked discrete diffusion.
\end{abstract}


\section{Introduction}\label{sec:introduction}
We
analyze a single static selection step in a stylized random-field model:
from a length-\(N\) locally dependent scalar score field, select the \(K\)
smallest values, and then evaluate a prescribed distance-dependent Gaussian
dependence cost (total correlation) on the selected indices.  The model contains no forward
corruption process, learned reverse transition, context evolution, remasking, or
multi-step schedule.  Its purpose is to isolate a stochastic-geometry question
suggested by confidence-based parallel decoding:

\begin{quote}
In one selection step, how quickly may the budget \(K_N\) grow with sequence
length \(N\) while the selected block has negligible dependence cost in the
stylized Gaussian model?
\end{quote}

The motivation comes from the factorization used by parallel categorical
decoders.  Conditional on the current masked context, let \(P_S\) be the target
joint categorical law on a selected set \(S\), let \(P_i\) be its one-position
marginals, and let \(Q_i\) be the one-position laws used by a product decoder.
Whenever the divergences are finite,
\begin{equation}\label{eq:intro-kl-decomposition}
    D_{\mathrm{KL}}\!\left(P_S\,\middle\|\,\bigotimes_{i\in S}Q_i\right)
    =
    \underbrace{D_{\mathrm{KL}}\!\left(
      P_S\,\middle\|\,\bigotimes_{i\in S}P_i
    \right)}_{\text{categorical total correlation}}
    +
    \sum_{i\in S}D_{\mathrm{KL}}(P_i\|Q_i).
\end{equation}
Thus accurate marginal predictions do not by themselves justify a parallel
product update \citep{watanabe1960information,coverthomas2006elements}.
Equation~\eqref{eq:intro-kl-decomposition} concerns the actual categorical
decoder.  The Gaussian quantity analyzed below is only an analytically tractable
cost motivated by its dependence term.  We neither derive that Gaussian cost
from a trained categorical decoder nor assume that the two total correlations
are numerically equal.

Discrete diffusion replaces continuous Gaussian noising by categorical,
absorbing, masking, or continuous-time jump processes
\citep{austin2021structured,hoogeboom2021argmax,campbell2022continuous,
hoogeboom2021ardm}.  Masked variants permit out-of-order and parallel generation
of tokenized images and language
\citep{chang2022maskgit,chang2023muse,lou2024sedd,sahoo2024mdlm}.
Large diffusion language models have made schedule design operational at scale,
and recent work analyzes token ordering, confidence-based unmasking, learned
policies, and parallel budgets
\citep{nie2025llada,kim2025tokenordering,li2025provable,
chen2026optimal,hong2026policies,jazbec2026unmasking}.  Those papers study
components of actual reverse-time samplers.  The present paper instead uses
that literature as motivation for a one-step model of score geometry and
dependence cost.

A common operational rule is to unmask positions with high marginal confidence.
For a categorical prediction \(p_i\), one possible scalar uncertainty summary is
the trace of its one-hot covariance,
\[
    C_i=\operatorname{diag}(p_i)-p_ip_i^\top,
    \qquad
    \operatorname{tr}(C_i)=1-\|p_i\|_2^2.
\]
This example motivates an ordering only.  We do not identify the vocabulary-
dimensional covariance \(C_i\) with the Gaussian covariance introduced below,
and we do not claim that the resulting score field is calibrated to any specific
neural confidence measure.  The modeling assumption is simply that smaller
values of a scalar score correspond to positions that a one-step rule would rank
as more confident.

Fix an integer \(m\ge1\) and two correlation parameters
\(0<|\rhoV|<1\) and \(0<|\rhoX|<1\).  Let \((Y_i)_{i\ge1}\) be a stationary
\(\mathbb R^m\)-valued Gaussian AR(1) field with parameter \(\rhoV\), and set
\begin{equation}\label{eq:intro-variance-field}
    V_i=\|Y_i\|^2.
\end{equation}
The field \((V_i)\) supplies the ordering scores; \(\rhoV\) controls clustering
of unusually small scores.  Separately, for the first \(N\) positions define
\begin{equation}\label{eq:intro-random-covariance}
    (\Sigma_N)_{ij}
    =\sqrt{V_iV_j}\,\rhoX^{|i-j|},
    \qquad 1\le i,j\le N.
\end{equation}
The indices of the \(K\) smallest diagonal entries form \(S_{N,K}\).  An
auxiliary vector \(X^{(N)}\mid\Sigma_N\sim N(0,\Sigma_N)\) is then used solely
to assign a Gaussian dependence cost to those indices.  The construction is
\begin{equation}\label{eq:intro-hierarchy}
    (Y_i)_{i=1}^N
    \longrightarrow
    (V_i)_{i=1}^N
    \longrightarrow
    \Sigma_N,
    \qquad
    \Sigma_N\longrightarrow S_{N,K},
    \qquad
    \Sigma_N\longrightarrow\mathcal L(X^{(N)}\mid\Sigma_N).
\end{equation}
The two parameters \(\rhoV\) and \(\rhoX\) are externally specified and need
not arise from a common categorical distribution.  This separation makes the
probability calculation transparent, but it is also why the construction should
be read as a stylized one-step random-field model rather than as a derived model
of masked diffusion.

\begin{table}[t]
\centering
\caption{Objects in the one-step random-field model.}
\label{tab:model-objects}
\begin{tabularx}{\textwidth}{@{}p{0.18\textwidth}XX@{}}
\toprule
Object & Mathematical role & Scope of the interpretation \\
\midrule
\(p_i,C_i\) & Categorical prediction and one-hot covariance & Motivate a possible scalar ranking only; they are not part of the Gaussian model. \\
\(Y_i,V_i=\|Y_i\|^2\) & Locally dependent score field & Determine the lowest-score indices; \(\rhoV\) controls lower-tail clustering. \\
\(\Sigma_N\) & Gaussian gap-cost representation & Attaches dependence through \(\rhoX\); it is not derived from a categorical decoder. \\
\(S_{N,K}\) & Indices of the \(K\) smallest \(V_i\) & Output of the single static selection step. \\
\(X^{(N)}\mid\Sigma_N\) & Auxiliary Gaussian vector & Defines the conditional Gaussian total-correlation cost on the selected set. \\
\bottomrule
\end{tabularx}
\end{table}

For a deterministic set \(S=\{i_1<\cdots<i_k\}\), write
\(X_S=(X_{i_1},\ldots,X_{i_k})\).  The model's conditional Gaussian total
correlation is
\begin{equation}\label{eq:intro-tc-definition}
    \TC_{\Sigma_N}(S)
    :=
    D_{\mathrm{KL}}\!\left(
      \mathcal L(X_S\mid\Sigma_N)
      \,\middle\|\,
      \bigotimes_{i\in S}\mathcal L(X_i\mid\Sigma_N)
    \right).
\end{equation}
For random \(S_{N,K}\), this is a random variable on the outer probability
space generated by \((Y_i)\).  All convergence in probability statements refer
to this outer randomness.

Once \(S\) is fixed, standardization removes the diagonal values \(V_i\) from
this Gaussian total correlation.  Consequently, a smaller score does not by
itself reduce the Gaussian dependence cost: score magnitudes matter only through
the ranking and the resulting spacings of the selected indices.  Stripped to
its probabilistic core, the paper studies rare minima of a short-range dependent
score field together with a decreasing gap-dependent cost. 

The score field is analytically tractable because \(V_i\sim\chi_m^2\), and very
small scores are dependent Gaussian small-ball events.  If a threshold keeps a
marginal fraction \(q\), the expected number of selected pairs within distance
\(L\) is of order \(NLq^2\).  With \(q\approx K/N\), this gives the collision
scale \(LK^2/N\).  The first main theorem converts spatial separation into an
upper bound on the prescribed Gaussian total-correlation cost.

\Needspace{16\baselineskip}
\begin{theorem}[First main result, readable form]\label{thm:main-readable}
Within the one-step random-field model, for every polynomial budget
\[
    K_N=\lfloor N^\alpha\rfloor,
    \qquad 0<\alpha<\frac12,
\]
the conditional Gaussian factorization cost on the selected indices vanishes:
\begin{equation}\label{eq:readable-main-conclusion}
    \TC_{\Sigma_N}(S_{N,K_N})\longrightarrow0
    \qquad\text{in probability.}
\end{equation}
More generally, the same conclusion holds whenever there is an integer distance
scale \(L_N\ge1\) such that
\begin{equation}\label{eq:readable-main-conditions}
    \frac{L_NK_N^2}{N}\longrightarrow0,
    \qquad
    K_N|\rhoX|^{2(L_N+1)}\longrightarrow0.
\end{equation}
The fully quantified statement is Theorem~\ref{thm:vanishing-tc} in
Section~\ref{sec:model}; its proof is given in
Appendix~\ref{app:first-main-proof}.
\end{theorem}

Theorem~\ref{thm:main-readable} states that a sufficiently conservative lowest-score set can
grow with \(N\) without accumulating a non-negligible value of the prescribed
Gaussian dependence cost.

The second result describes the model's critical collision scale.  Define
\begin{equation}\label{eq:intro-critical-constants}
    \vartheta_{\rhoV,m}:=(1-\rhoV^2)^{-m/2},
    \qquad
    c_{\rhoX}:=-\frac12\log(1-\rhoX^2)>0.
\end{equation}

\Needspace{13\baselineskip}
\begin{theorem}[Second main result, readable form]
\label{thm:critical-readable}
Within the same one-step model, if
\[
    \frac{K_N}{\sqrt N}\longrightarrow\lambda\in(0,\infty),
\]
then the conditional Gaussian total-correlation cost does not vanish.  More
precisely,
\begin{equation}\label{eq:intro-critical-obstruction}
    \liminf_{N\to\infty}
    \bP\!\left(
      \TC_{\Sigma_N}(S_{N,K_N})\ge c_{\rhoX}
    \right)
    \ge
    1-\exp\{-\vartheta_{\rhoV,m}\lambda^2\}>0,
\end{equation}
and
\begin{equation}\label{eq:intro-critical-expectation}
    \liminf_{N\to\infty}
    \bE\!\left[\TC_{\Sigma_N}(S_{N,K_N})\right]
    \ge
    c_{\rhoX}\vartheta_{\rhoV,m}\lambda^2>0.
\end{equation}
The fully quantified result is Theorem~\ref{thm:critical-tc} in
Section~\ref{sec:critical}.
\end{theorem}

The exact square-root limit is proved for adjacent selected pairs; it yields the
one-sided lower bounds above because each adjacent pair contributes a fixed
positive amount to the Gaussian gap cost.

\paragraph{Contributions.}
\begin{enumerate}
    \item We formulate a transparent one-step stochastic-geometry model
    motivated by confidence-based unmasking, while explicitly separating it
    from a full masked-diffusion process.
    \item For exact lowest-score selection, we prove a conservative regime in
    which the prescribed conditional Gaussian total-correlation cost vanishes.
    \item At the square-root collision scale, we prove one-sided probability
    and expectation lower bounds showing that this Gaussian cost is
    non-vanishing.
    \item We identify the auxiliary geometry of rare selected minima and verify
    the model's asymptotic predictions in synthetic simulations of the same
    random field.
\end{enumerate}

Section~\ref{sec:model} formalizes the one-step random-covariance model and
states the first main theorem.  Section~\ref{sec:applications} records internal
model implications and simulations.  Section~\ref{sec:subcritical} states the
auxiliary sparsification theorem and the proof roadmap.  Section~\ref{sec:critical}
states the critical lower-bound theorem.  Appendices~\ref{app:first-main-proof}
and \ref{app:critical-poisson-proof} contain the technical proofs.  The final
sections discuss the relation to masked-diffusion scheduling, the model's scope,
and its limitations.

\section{Stylized one-step random-covariance model and the first main theorem}\label{sec:model}

For \(N\in\mathbb N\), write \([N]=\{1,\ldots,N\}\), and let \(I_m\)
denote the \(m\times m\) identity matrix.  The symbols \(m\), \(\rhoV\),
\(\rhoX\), \((Y_i)\), \((V_i)\), \(\Sigma_N\), \(S_{N,K}\), and
\(\TC_{\Sigma_N}(S)\) retain the meanings introduced in
Section~\ref{sec:introduction}.  All asymptotic statements are taken as
\(N\to\infty\), with \(m\), \(\rhoV\), and \(\rhoX\) fixed.  This section supplies the formal probability-space construction and
introduces the auxiliary notation used in the proofs.  The construction is a
static one-step model; it is not a reverse-time transition kernel and does not
specify how the score field would evolve after selected positions are updated.

\subsection{Latent variance field and covariance construction}

Realize the stationary field introduced in Section~\ref{sec:introduction}
as the \(\mathbb R^m\)-valued Gaussian AR(1) chain
\begin{equation}\label{eq:uncertainty-field-ar1}
    Y_i=\rhoV Y_{i-1}+\sqrt{1-\rhoV^2}\,\xi_i,
    \qquad i\ge2,
\end{equation}
where \(Y_1\sim N(0,I_m)\), the innovations are independent
\(N(0,I_m)\), and \(Y_1\) is independent of them.  The variance profile introduced in \eqref{eq:intro-variance-field} is
therefore \(V_i=\|Y_i\|^2\); it is a covariance diagonal, not an estimator
computed from repeated samples.

For later matrix calculations, introduce the auxiliary notation
\begin{equation}\label{eq:random-covariance-definition}
    D_{V,N}=\operatorname{diag}(V_1,\ldots,V_N),
    \qquad
    R_{\rhoX,N}=(\rhoX^{|i-j|})_{i,j=1}^N.
\end{equation}
The entrywise covariance specified in
\eqref{eq:intro-random-covariance} then has the factorization
\begin{equation}\label{eq:Sigma-random-environment}
    \Sigma_N=D_{V,N}^{1/2}R_{\rhoX,N}D_{V,N}^{1/2}.
\end{equation}

\begin{proposition}[Validity of the random-covariance construction]\label{prop:covariance-environment}
Almost surely, \(V_i>0\) for every \(i\), \(\Sigma_N\) is positive definite,
and
\begin{equation}\label{eq:Sigma-diagonal}
    (\Sigma_N)_{ii}=V_i.
\end{equation}
Moreover, on an extension of the probability space, let \((Z_i)_{i\ge1}\) be
a scalar stationary Gaussian AR(1) chain with correlation parameter \(\rhoX\),
independent of \((Y_i)_{i\ge1}\), and set
\begin{equation}\label{eq:explicit-conditional-process}
    X_i=\sqrt{V_i}\,Z_i,
    \qquad i\ge1.
\end{equation}
Then, for every \(N\),
\begin{equation}\label{eq:conditional-process}
    X^{(N)}=(X_1,\ldots,X_N)\mid\Sigma_N
    \sim N(0,\Sigma_N).
\end{equation}
\end{proposition}

\begin{proof}
Each \(Y_i\) has a nondegenerate Gaussian density on \(\mathbb R^m\), so
\(\bP(Y_i=0)=0\).  A countable union shows that almost surely \(V_i>0\) for
all \(i\).  The conditional-correlation matrix \(R_{\rhoX,N}\) is positive definite
for \(|\rhoX|<1\).  On the event that all \(V_i>0\), \(D_{V,N}^{1/2}\) is
invertible, and the congruence
\(D_{V,N}^{1/2}R_{\rhoX,N}D_{V,N}^{1/2}\) is positive definite.  Its diagonal is
\(V_i(R_{\rhoX,N})_{ii}=V_i\).

Because \((Z_i)\) is independent of the variance field, its conditional law
given \(\Sigma_N\) is still centered Gaussian with covariance
\(R_{\rhoX,N}\).  Equation~\eqref{eq:explicit-conditional-process} therefore
gives
\[
    \operatorname{Cov}(X^{(N)}\mid\Sigma_N)
    =D_{V,N}^{1/2}R_{\rhoX,N}D_{V,N}^{1/2}
    =\Sigma_N,
\]
which proves \eqref{eq:conditional-process}.
\end{proof}

The construction is static.  First sample the score field and hence
\(\Sigma_N\); next read its diagonal and determine the selected set; only then,
if desired, use the independent chain \((Z_i)\) to realize the auxiliary
Gaussian vector.  The selection results depend on the random covariance but not
on a realization of \((X_i)\).  No claim is made that the pair
\((V_i,R_{
hoX,N})\) is induced by a common categorical decoder.

\begin{remark}[Rank-preserving calibration invariance]\label{rem:monotone-calibration}
Let \(g:(0,\infty)\to(0,\infty)\) be strictly increasing and define
\[
    \widetilde V_i=g(V_i),
    \qquad
    D_{\widetilde V,N}=\operatorname{diag}(\widetilde V_1,\ldots,\widetilde V_N),
\]
and define
\[
    \widetilde\Sigma_N
    =D_{\widetilde V,N}^{1/2}R_{\rhoX,N}D_{\widetilde V,N}^{1/2}.
\]
Then the exact lowest-\(K\) set is unchanged, because \(g\) preserves the score
ordering, and standardization removes the transformed diagonal scale from every
retained conditional Gaussian subvector.  Consequently all spacing, Poisson, and total-
correlation conclusions remain valid.  In particular, a bounded increasing
calibration may be used when matching the latent chi-square score to a bounded
categorical uncertainty measure.
\end{remark}

\subsection{Chi-square small-ball structure}

\begin{proposition}[Dependent chi-square variance field]\label{prop:chi-square-field}
The process \((Y_i)\) is centered Gaussian with
\begin{equation}\label{eq:Y-covariance}
    \Cov(Y_i,Y_j)=\rhoV^{|i-j|}I_m.
\end{equation}
Each diagonal variance satisfies \(V_i\sim\chi_m^2\), and
\begin{equation}\label{eq:D-covariance}
    \Cov(V_i,V_j)=2m\rhoV^{2|i-j|}.
\end{equation}
\end{proposition}

\begin{proof}
Iterating \eqref{eq:uncertainty-field-ar1} gives
\[
    Y_j=\rhoV^{j-i}Y_i+
    \sqrt{1-\rhoV^2}\sum_{\ell=i+1}^j\rhoV^{j-\ell}\xi_\ell,
    \qquad j>i.
\]
The innovation sum is independent of \(Y_i\), which proves
\eqref{eq:Y-covariance}; stationarity gives \(Y_i\sim N(0,I_m)\), so
\(V_i=\|Y_i\|^2\sim\chi_m^2\).  Isserlis' formula yields
\[
    \Cov(V_i,V_j)
    =2\sum_{a,b=1}^m\Cov(Y_{i,a},Y_{j,b})^2
    =2m\rhoV^{2|i-j|}.
\]
\end{proof}

Let
\begin{equation}\label{eq:chi-cdf}
    F_m(u)=\bP(\chi_m^2\le u),
\end{equation}
and define
\begin{equation}\label{eq:quantile}
    u_m(q)=F_m^{-1}(q),
    \qquad q\in(0,1).
\end{equation}
As \(u\downarrow0\),
\begin{equation}\label{eq:chi-small-u}
    F_m(u)
    =
    \frac{u^{m/2}}{2^{m/2}\Gamma(m/2+1)}(1+O(u)).
\end{equation}

\subsection{Exact lowest-\texorpdfstring{$K$}{K} variance selection}

Recall that \(S_{N,K}\) denotes the exact lowest-variance set introduced in
Section~\ref{sec:introduction}.  This set is almost surely well-defined without
ties: for \(i\ne j\), the vector \((Y_i,Y_j)\) has a nondegenerate Gaussian
density and \(\|Y_i\|^2-\|Y_j\|^2\) is a nonzero polynomial, whose zero set
has Lebesgue measure zero.

\begin{definition}[Threshold set]\label{def:topk-threshold-sets}
For \(q\in(0,1)\), define
\begin{equation}\label{eq:threshold-set}
    T_N(q)=\{i\in[N]:V_i\le u_m(q)\}.
\end{equation}
Then
\begin{equation}\label{eq:threshold-mean}
    \bE|T_N(q)|=Nq.
\end{equation}
\end{definition}

For distance \(h\ge1\), set
\begin{equation}\label{eq:pair-prob}
    p_h(q)=\bP(V_1\le u_m(q),\ V_{1+h}\le u_m(q)).
\end{equation}

\begin{lemma}[Fixed-distance two-point small balls]\label{lem:fixed-distance-small-ball}
For every fixed \(h\ge1\),
\begin{equation}\label{eq:pair-asymptotic}
    p_h(q)
    =
    (1-\rhoV^{2h})^{-m/2}q^2(1+o(1)),
    \qquad q\downarrow0.
\end{equation}
\end{lemma}

\begin{proof}
Put \(r=\rhoV^h\) and \(u=u_m(q)\).  The joint density of
\((Y_1,Y_{1+h})\) at the origin is
\[
    (2\pi)^{-m}(1-r^2)^{-m/2}.
\]
Continuity at the origin gives
\[
    \bP(\|Y_1\|^2\le u,\|Y_{1+h}\|^2\le u)
    =
    (2\pi)^{-m}(1-r^2)^{-m/2}v_m^2u^m(1+o(1)),
\]
where \(v_m\) is the volume of the unit ball in \(\mathbb R^m\).  Similarly,
\[
    F_m(u)=(2\pi)^{-m/2}v_mu^{m/2}(1+o(1)).
\]
Dividing by \(F_m(u)^2=q^2\) proves the claim.
\end{proof}

\begin{definition}[Collisions and spacing]\label{def:pair-counts-model-section}
For \(S\subset[N]\) and \(L\ge1\), define
\begin{equation}\label{eq:collision-count}
    C_L(S)=
    \sum_{h=1}^{\min\{L,N-1\}}
    \sum_{i=1}^{N-h}
    \mathbf 1_{\{i\in S\}}\mathbf 1_{\{i+h\in S\}},
\end{equation}
\begin{equation}\label{eq:adjacent-count}
    A(S)=C_1(S),
\end{equation}
and
\begin{equation}\label{eq:min-spacing}
    \Delta(S)=\min_{i\ne j,\ i,j\in S}|i-j|,
\end{equation}
with \(\Delta(S)=\infty\) for \(|S|\le1\).  Then
\begin{equation}\label{eq:spacing-collision-equivalence}
    \Delta(S)>L
    \quad\Longleftrightarrow\quad
    C_L(S)=0.
\end{equation}
\end{definition}

\subsection{Conditional total correlation}

Recall the conditional total-correlation notation from
\eqref{eq:intro-tc-definition}.  For a deterministic set
\(S=\{i_1<\cdots<i_k\}\subset[N]\), standardization by the conditional
marginal standard deviations \(\sqrt{V_i}\) removes the random diagonal scale.
Hence the correlation matrix of \(X_S\mid\Sigma_N\) is
\begin{equation}\label{eq:retained-correlation}
    R_{\rhoX}(S)=(\rhoX^{|i_a-i_b|})_{a,b=1}^k.
\end{equation}
The Gaussian KL formula therefore gives the closed form
\begin{equation}\label{eq:total-correlation}
    \TC_{\Sigma_N}(S)
    =
    -\frac12\log\det R_{\rhoX}(S).
\end{equation}
This identity applies pointwise to every realized covariance and every
\(\Sigma_N\)-measurable selected set.

\begin{lemma}[Exact AR(1) determinant]\label{lem:exact-ar1-determinant}
If \(S=\{i_1<\cdots<i_k\}\), then
\begin{equation}\label{eq:exact-ar1-determinant}
    \det R_{\rhoX}(S)
    =
    \prod_{a=1}^{k-1}
    \left(1-\rhoX^{2(i_{a+1}-i_a)}\right).
\end{equation}
\end{lemma}

\begin{proof}
The subsampled stationary AR(1) chain with parameter \(\rhoX\) satisfies
\[
    Z_{i_{a+1}}
    =
    \rhoX^{i_{a+1}-i_a}Z_{i_a}
    +\sqrt{1-\rhoX^{2(i_{a+1}-i_a)}}\,\eta_a,
\]
with independent standard Gaussian innovations \(\eta_a\).  The covariance
determinant is the product of the first marginal variance, equal to one, and
the successive conditional variances.
\end{proof}

\begin{lemma}[Separated sets have small total correlation]\label{lem:tc-bound}
Let \(|S|=k\ge2\) and suppose \(\Delta(S)>L\).  Then
\begin{equation}\label{eq:tc-bound}
    0\le \TC_{\Sigma_N}(S)
    \le
    \frac{k-1}{2}\,
    \frac{|\rhoX|^{2(L+1)}}{1-|\rhoX|^{2(L+1)}}.
\end{equation}
\end{lemma}

\begin{proof}
Write \(S=\{i_1<\cdots<i_k\}\) and
\(g_a=i_{a+1}-i_a\).  By Lemma~\ref{lem:exact-ar1-determinant},
\[
    \TC_{\Sigma_N}(S)
    =
    -\frac12\sum_{a=1}^{k-1}\log\bigl(1-\rhoX^{2g_a}\bigr).
\]
The spacing assumption gives \(g_a\ge L+1\), and hence
\(0\le \rhoX^{2g_a}\le |\rhoX|^{2(L+1)}<1\).  Using
\(-\log(1-x)\le x/(1-x)\) for \(0\le x<1\), each summand is at most
\[
    \frac12\,
    \frac{|\rhoX|^{2(L+1)}}{1-|\rhoX|^{2(L+1)}}.
\]
Summing over \(a=1,\ldots,k-1\) proves \eqref{eq:tc-bound}.
\end{proof}

\Needspace{16\baselineskip}
\begin{theorem}[First main theorem: subcritical vanishing of conditional total correlation]
\label{thm:vanishing-tc}
Let \((K_N)_{N\ge1}\) and \((L_N)_{N\ge1}\) be integer sequences satisfying
\[
    1\le K_N\le N,
    \qquad
    L_N\ge1.
\]
Assume
\begin{equation}\label{eq:subcritical-assumption-main}
    \frac{L_NK_N^2}{N}\longrightarrow0
\end{equation}
and
\begin{equation}\label{eq:decorrelation-assumption-section3}
    K_N|\rhoX|^{2(L_N+1)}\longrightarrow0.
\end{equation}
Then the exact lowest-variance set \(S_{N,K_N}\) satisfies
\begin{equation}\label{eq:tc-convergence-probability}
    \TC_{\Sigma_N}(S_{N,K_N})\longrightarrow0
    \qquad\text{in probability,}
\end{equation}
where probability is taken over the latent variance field \((Y_i)\), equivalently
over the realized random covariance \(\Sigma_N\).
\end{theorem}

The proof is given in Appendix~\ref{app:first-main-proof}.  Its essential
geometric input is the auxiliary sparsification theorem stated in
Section~\ref{sec:subcritical}.

\section{Internal model interpretation and numerical illustration}\label{sec:applications}

This section interprets and checks the asymptotic statements only within the
stylized one-step model. 

\subsection{Budget implications within the model}

The auxiliary sparsification theorem gives an internal geometric criterion.  To
make the probability of selecting two positions within radius \(L_N\) vanish, it
is sufficient that
\[
    K_N=o\!\left(\sqrt{\frac{N}{L_N}}\right).
\]
For the auxiliary Gaussian coordinates to have vanishing conditional total
correlation, the exact AR(1) determinant further requires
\[
    K_N|\rhoX|^{2(L_N+1)}\to0.
\]
For \(K_N=\lfloor N^\alpha\rfloor\), \(0<\alpha<1/2\), a logarithmic
separation \(L_N=\lceil c\log N\rceil\) with
\[
    c>\frac{\alpha}{2\log(1/|\rhoX|)}
\]
satisfies both conditions.

At \(K_N\sim\lambda\sqrt N\), Theorem~\ref{thm:critical-tc} gives a
non-vanishing lower bound for the model's Gaussian gap cost.  Its auxiliary
Poisson theorem identifies the corresponding probability of at least one
adjacent selected pair.  Every fixed positive value of \(K_N/\sqrt N\) belongs
to this critical window; \(\lambda=1\) is not a distinguished boundary. 

\subsection{Synthetic checks of the random-field model}

We simulate the same mathematical model with
\[
    m=3,
    \qquad
    \rhoV=\rhoX=0.5.
\]
For each realization, the exact lowest-score set is selected from
\(V_i=\|Y_i\|^2\).  If
\(S_{N,K}=\{s_1<\cdots<s_K\}\), the prescribed Gaussian total-correlation
cost is
\begin{equation}\label{eq:simulation-tc}
    \TC_{\Sigma_N}(S_{N,K})
    =
    -\frac12\sum_{a=1}^{K-1}
    \log\!\left(1-\rhoX^{2(s_{a+1}-s_a)}\right).
\end{equation}
We use \(N\in\{1024,4096,16384\}\) and 500 independent score-field
realizations per sequence length.

Figure~\ref{fig:direct-covariance-tc} plots the mean Gaussian cost on a
logarithmic vertical axis against \(K/\sqrt N\).  The approximate collapse of
the curves supports \(K/\sqrt N\) as the finite-size scaling variable in this
model.  It does not establish a categorical factorization error, a neural
unmasking threshold, or a generation-quality barrier.

\begin{figure}[!b]
    \centering
    \includegraphics[width=0.82\textwidth]
    {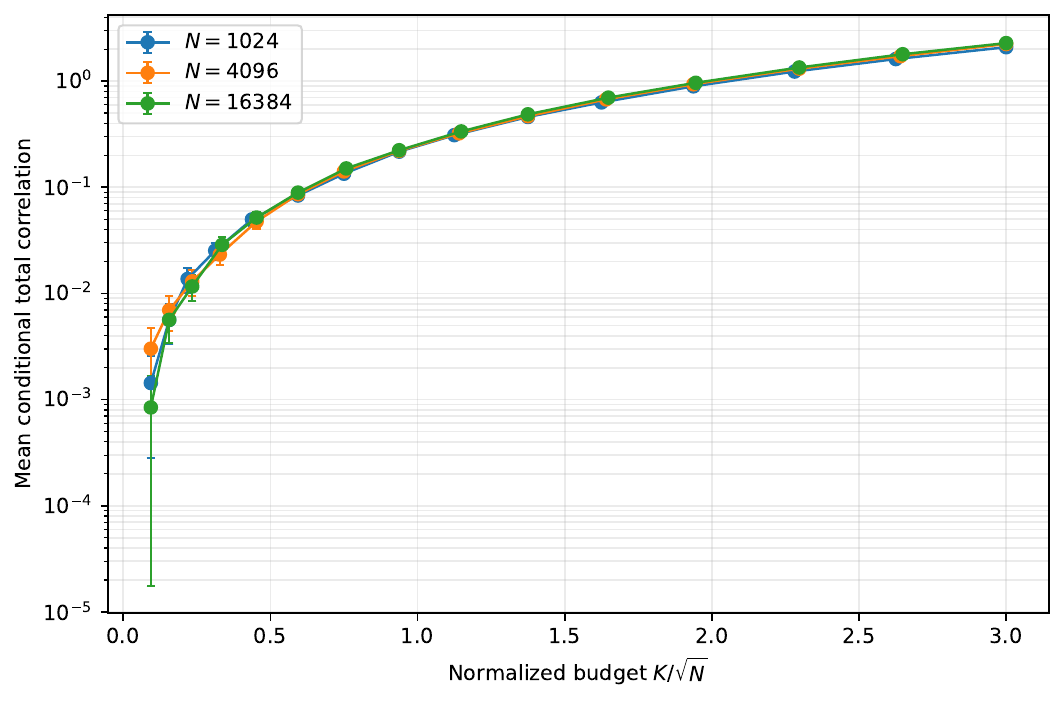}
    \caption{Mean conditional Gaussian total-correlation cost in the stylized
    one-step model.  Here \(m=3\), \(\rhoV=\rhoX=0.5\),
    \(N\in\{1024,4096,16384\}\), and each point uses 500 independent
    score-field realizations.  The vertical axis is logarithmic, and error bars
    are 95\% Monte Carlo confidence intervals for the mean.  The figure checks
    the internal scaling of the model, not a trained masked-diffusion decoder.}
    \label{fig:direct-covariance-tc}
\end{figure}

Figure~\ref{fig:direct-covariance-adjacent} checks the auxiliary critical theorem
more directly.  The empirical probability of no adjacent selected pair is
compared with the Poisson prediction, evaluated at \(\lambda=K/\sqrt N\).
This is an internal asymptotic check of the score field.

\begin{figure}[!t]
    \centering
    \includegraphics[width=0.82\textwidth]
    {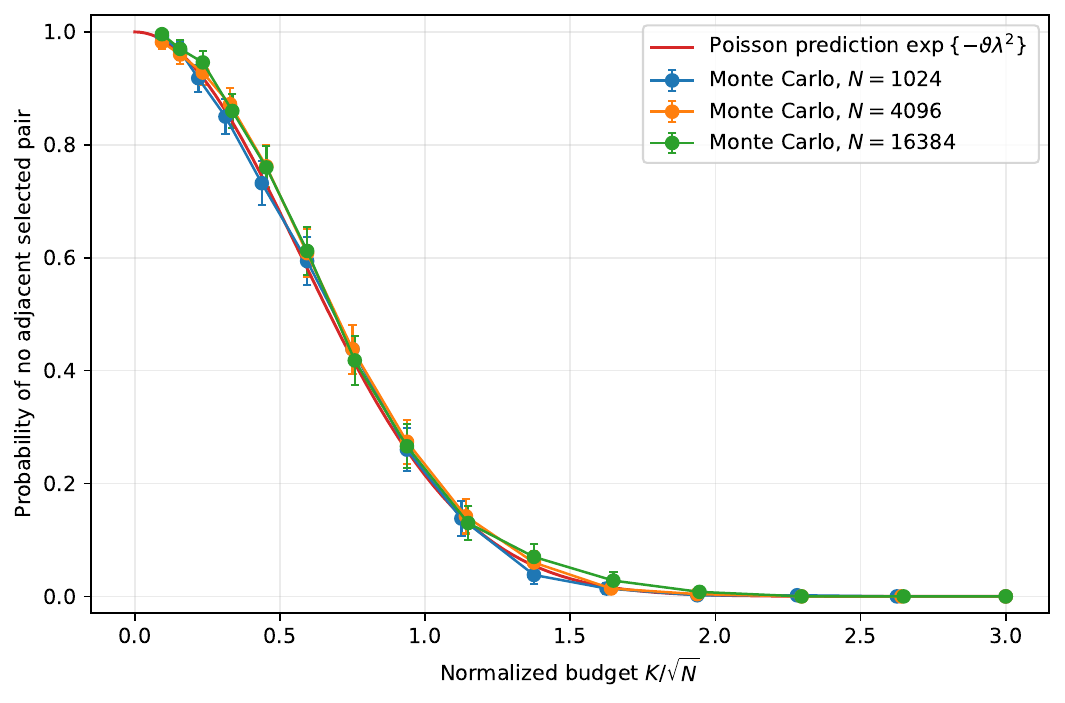}
    \caption{Probability of no adjacent selected pair in the stylized score
    model for \(m=3\), \(\rhoV=\rhoX=0.5\),
    \(N\in\{1024,4096,16384\}\), and 500 independent realizations per
    sequence length.  Solid points show Monte Carlo estimates with 95\%
    confidence intervals; the continuous curve is the auxiliary Poisson
    prediction \(\exp\{-\vartheta_{\rhoV,m}\lambda^2\}\).}
    \label{fig:direct-covariance-adjacent}
\end{figure}
\FloatBarrier

\section{Auxiliary sparsification for the one-step score model and proof roadmap}\label{sec:subcritical}

The first main theorem is proved through the following static geometric
statement.  It concerns only the locations of rare minima of the score field.
Its role is to show that the selected indices are separated far enough for the
prescribed AR(1) Gaussian gap cost to vanish.

\begin{theorem}[Auxiliary sparsification theorem]\label{thm:subcritical-spacing}
Let \((K_N)_{N\ge1}\) and \((L_N)_{N\ge1}\) be integer sequences satisfying
\[
    1\le K_N\le N,
    \qquad
    L_N\ge1,
\]
and
\begin{equation}\label{eq:subcritical-assumption}
    \frac{L_NK_N^2}{N}\longrightarrow0.
\end{equation}
Then
\begin{equation}\label{eq:subcritical-collision-result}
    \bP\bigl(C_{L_N}(S_{N,K_N})=0\bigr)\longrightarrow1,
\end{equation}
or equivalently
\begin{equation}\label{eq:subcritical-spacing-result}
    \bP\bigl(\Delta(S_{N,K_N})>L_N\bigr)\longrightarrow1.
\end{equation}
In particular, if \(K_N=o(\sqrt N)\), then
\begin{equation}\label{eq:no-adjacent-result}
    \bP\bigl(A(S_{N,K_N})=0\bigr)\longrightarrow1.
\end{equation}
\end{theorem}

Theorem~\ref{thm:subcritical-spacing} is the geometric input to the first main
result.  On the event \(\Delta(S_{N,K_N})>L_N\), the exact AR(1) determinant
bound from Lemma~\ref{lem:tc-bound} gives
\begin{equation}\label{eq:main-text-tc-roadmap}
    0\le \TC_{\Sigma_N}(S_{N,K_N})
    \le
    \frac{K_N-1}{2}\,
    \frac{|\rhoX|^{2(L_N+1)}}{1-|\rhoX|^{2(L_N+1)}}.
\end{equation}
Thus the sparsification probability in
\eqref{eq:subcritical-spacing-result}, together with the correlation-decay
assumption of Theorem~\ref{thm:vanishing-tc}, forces the conditional total
correlation to vanish.  Appendix~\ref{app:first-main-proof} contains the full
proof of the auxiliary sparsification theorem and the formal deduction of the
first main theorem.

\section{Critical non-vanishing in the one-step Gaussian cost model}
\label{sec:critical}

The first main theorem gives a regime in which the model's Gaussian cost
vanishes.  This section proves the complementary statement that, at every fixed
positive square-root budget, the same cost is non-negligible.  The conclusion is
internal to the one-step model and is not a lower bound on categorical decoding
error or on per-token generation loss.  Recall the constants
\(\vartheta_{\rhoV,m}\) and \(c_{\rhoX}\) from
\eqref{eq:intro-critical-constants}.

\Needspace{18\baselineskip}
\begin{theorem}[Second main theorem: critical non-vanishing of conditional total correlation]
\label{thm:critical-tc}
Let \((K_N)_{N\ge1}\) be an integer sequence satisfying
\(1\le K_N\le N\) and
\[
    \frac{K_N}{\sqrt N}\longrightarrow\lambda\in(0,\infty).
\]
Let \(Z_\lambda\) have the Poisson distribution with mean
\(\vartheta_{\rhoV,m}\lambda^2\).  Then, for every integer \(r\ge1\),
\begin{equation}\label{eq:critical-tc-tail-bound}
    \liminf_{N\to\infty}
    \bP\!\left(
      \TC_{\Sigma_N}(S_{N,K_N})\ge r c_{\rhoX}
    \right)
    \ge
    \bP(Z_\lambda\ge r).
\end{equation}
In particular,
\begin{equation}\label{eq:critical-tc-obstruction}
    \liminf_{N\to\infty}
    \bP\!\left(
      \TC_{\Sigma_N}(S_{N,K_N})\ge c_{\rhoX}
    \right)
    \ge
    1-\exp\{-\vartheta_{\rhoV,m}\lambda^2\}>0,
\end{equation}
and
\begin{equation}\label{eq:critical-tc-expectation-bound}
    \liminf_{N\to\infty}
    \bE\!\left[\TC_{\Sigma_N}(S_{N,K_N})\right]
    \ge
    c_{\rhoX}\vartheta_{\rhoV,m}\lambda^2>0.
\end{equation}
Consequently,
\(\TC_{\Sigma_N}(S_{N,K_N})\) does not converge to zero in probability at a
fixed positive critical budget.
\end{theorem}

The proof uses the following adjacent-pair limit as an auxiliary theorem.

\Needspace{14\baselineskip}
\begin{theorem}[Auxiliary critical adjacent-pair limit]
\label{thm:exact-critical-poisson}
Let \((K_N)_{N\ge1}\) be an integer sequence satisfying
\(1\le K_N\le N\).  If
\[
    \frac{K_N}{\sqrt N}\longrightarrow\lambda\in(0,\infty),
\]
then
\begin{equation}\label{eq:exact-critical-poisson}
    A(S_{N,K_N})
    \Rightarrow
    \operatorname{Poisson}(\vartheta_{\rhoV,m}\lambda^2).
\end{equation}
In particular,
\begin{equation}\label{eq:exact-critical-zero-prob}
    \bP\bigl(A(S_{N,K_N})=0\bigr)
    \longrightarrow
    \exp\{-\vartheta_{\rhoV,m}\lambda^2\}.
\end{equation}
\end{theorem}

The proof of Theorem~\ref{thm:exact-critical-poisson}, including the factorial-moment argument and the two-sided threshold sandwich, is given in Appendix~\ref{app:critical-poisson-proof}.

\begin{proof}[Proof of Theorem~\ref{thm:critical-tc}]
Write a deterministic selected set as \(S=\{i_1<\cdots<i_k\}\).  By
Lemma~\ref{lem:exact-ar1-determinant},
\[
    \TC_{\Sigma_N}(S)
    =-\frac12\sum_{a=1}^{k-1}
      \log\!\left(1-\rhoX^{2(i_{a+1}-i_a)}\right).
\]
Every index \(a\) with \(i_{a+1}-i_a=1\) contributes exactly
\(c_{\rhoX}\), and all other summands are nonnegative.  Hence
\begin{equation}\label{eq:tc-dominates-adjacencies}
    \TC_{\Sigma_N}(S)\ge c_{\rhoX}A(S).
\end{equation}
For every integer \(r\ge1\),
\[
    \{A(S_{N,K_N})\ge r\}
    \subseteq
    \{\TC_{\Sigma_N}(S_{N,K_N})\ge r c_{\rhoX}\}.
\]
Theorem~\ref{thm:exact-critical-poisson} implies
\[
    \bP(A(S_{N,K_N})\ge r)
    \longrightarrow
    \bP(Z_\lambda\ge r),
\]
which proves \eqref{eq:critical-tc-tail-bound} and, by taking \(r=1\),
\eqref{eq:critical-tc-obstruction}.

For the expectation bound, \eqref{eq:tc-dominates-adjacencies} gives
\[
    \bE[\TC_{\Sigma_N}(S_{N,K_N})]
    \ge
    c_{\rhoX}\bE[A(S_{N,K_N})].
\]
For each \(M>0\), the function \(x\mapsto x\wedge M\) is bounded and
continuous on \(\mathbb N_0\).  Therefore the auxiliary Poisson convergence
gives
\[
    \lim_{N\to\infty}
    \bE[A(S_{N,K_N})\wedge M]
    =
    \bE[Z_\lambda\wedge M].
\]
Since \(A(S_{N,K_N})\ge A(S_{N,K_N})\wedge M\), first take the lower limit in
\(N\) and then let \(M\to\infty\).  Monotone convergence yields
\[
    \liminf_{N\to\infty}
    \bE[A(S_{N,K_N})]
    \ge
    \bE[Z_\lambda]
    =\vartheta_{\rhoV,m}\lambda^2.
\]
This proves \eqref{eq:critical-tc-expectation-bound}.
\end{proof}

\begin{corollary}[Adjacent-collision phase diagram]
\label{cor:adjacent-phase-diagram}
For the exact lowest-variance set \(S_{N,K_N}\):
\begin{enumerate}
    \item if \(K_N=o(\sqrt N)\), then
    \(\bP(A(S_{N,K_N})=0)\to1\);
    \item if \(K_N/\sqrt N\to\lambda\in(0,\infty)\), then
    \(A(S_{N,K_N})\Rightarrow
    \operatorname{Poisson}(\vartheta_{\rhoV,m}\lambda^2)\);
    \item if \(K_N/\sqrt N\to\infty\), then
    \begin{equation}\label{eq:phase-supercritical}
       \bP(A(S_{N,K_N})\ge1)\longrightarrow1.
    \end{equation}
\end{enumerate}
\end{corollary}

\begin{proof}[Proof of Corollary~\ref{cor:adjacent-phase-diagram}]
The first statement is Theorem~\ref{thm:subcritical-spacing} with \(L_N\equiv1\).
The second statement is Theorem~\ref{thm:exact-critical-poisson}.

For the third statement, fix any \(\lambda_0>0\) and define the auxiliary
budget \(\widetilde K_N=\lfloor\lambda_0\sqrt N\rfloor\).  Since
\(K_N/\sqrt N\to\infty\), we have \(\widetilde K_N\le K_N\) for all
sufficiently large \(N\).  By monotonicity of the order-statistic sets,
\begin{equation}\label{eq:topk-monotonicity}
    S_{N,\widetilde K_N}\subseteq S_{N,K_N}
\end{equation}
almost surely for all sufficiently large \(N\).  Therefore
\[
    \bP\bigl(A(S_{N,K_N})=0\bigr)
    \le
    \bP\bigl(A(S_{N,\widetilde K_N})=0\bigr).
\]
By Theorem~\ref{thm:exact-critical-poisson}, the right-hand side converges to
\(\exp\{-\vartheta_{\rhoV,m}\lambda_0^2\}\).  Taking \(\limsup\) and then letting
\(\lambda_0\to\infty\) gives
\[
    \limsup_{N\to\infty}\bP\bigl(A(S_{N,K_N})=0\bigr)=0,
\]
which is equivalent to \eqref{eq:phase-supercritical}.
\end{proof}

\begin{corollary}[Supercritical total-correlation obstruction]
\label{cor:supercritical-tc}
If \(K_N/\sqrt N\to\infty\), then
\begin{equation}\label{eq:supercritical-tc-obstruction}
    \bP\!\left(
      \TC_{\Sigma_N}(S_{N,K_N})\ge c_{\rhoX}
    \right)\longrightarrow1.
\end{equation}
\end{corollary}

\begin{proof}
Equation~\eqref{eq:tc-dominates-adjacencies} and
Corollary~\ref{cor:adjacent-phase-diagram} give
\[
\bP\!\left(
  \TC_{\Sigma_N}(S_{N,K_N})\ge c_{\rhoX}
\right)
\ge
\bP(A(S_{N,K_N})\ge1)
\longrightarrow1.
\]
\end{proof}

\begin{remark}[Scope of the square-root result]
The auxiliary collision results, proved in
Appendix~\ref{app:critical-poisson-proof}, identify
\(K\asymp\sqrt N\) as the exact
adjacent-collision scale.  Their principal role here is to prove the second main
total-correlation theorem, Theorem~\ref{thm:critical-tc}.  That theorem gives
explicit probability and expectation lower bounds at fixed positive critical
budgets, while Theorem~\ref{thm:vanishing-tc} gives a sufficient subcritical
regime for vanishing.  We do not claim an exact limiting distribution or a
complete phase diagram for the total-correlation statistic itself.
\end{remark}

\section{Discussion}\label{sec:discussion}

The mathematical content of the paper can be summarized without diffusion
terminology.  A locally dependent nonnegative score field is sampled, the
\(K\) smallest scores are selected, and the selected spacings are charged the
gap cost (total correlation)
\[
    f(h)=-\frac12\log(1-\rhoX^{2h}).
\]
The first theorem shows that this cost vanishes under a conservative budget and
correlation-decay regime.  The second shows that adjacent rare minima create a
non-vanishing one-sided obstruction at the square-root collision scale.  This is
a stochastic-geometry result about selected minima plus a prescribed decreasing
gap cost.

The Gaussian covariance representation provides a convenient information-
theoretic interpretation of that gap cost, but it should not be confused with a
categorical model derived from a masked decoder.  The score field and dependence
field have separate parameters, \(\rhoV\) and \(\rhoX\), and no common hidden
categorical distribution is specified.  Moreover, once the selected set is
known, diagonal score magnitudes cancel from Gaussian total correlation.  Low
scores affect the conclusion only by changing where selected indices occur; low
marginal variance does not itself make a fixed Gaussian pair less dependent.

Actual scheduling
work studies reverse-time token ordering, learned unmasking policies, and global
multi-step errors \citep{kim2025tokenordering,li2025provable,
chen2026optimal,hong2026policies,jazbec2026unmasking}.  The present model offers
a tractable baseline for a single confidence-ranked selection step: it isolates
how short-range lower-tail clustering can create nearby selections and how an
externally specified distance-dependent cost responds.  Likewise, its relation
to Gaussian explanations of diffusion is more conceptual 
\citep{wang2024gaussian,sahoo2025duality}.

\paragraph{Limitations.}
The analysis treats one static selection step, assumes a chi-square short-range
score field, and attaches an AR(1) Gaussian dependence cost that is not derived
from the same categorical law as the confidence score.  It does not prove that
scores from a trained masked model have this lower-tail geometry, that neural
token dependence decays with distance as assumed, or that Gaussian total
correlation approximates categorical factorization error. These restrictions make the result a transparent
null model and suggest concrete empirical questions: measure lower-tail score
collisions and conditional dependence in trained models, and then test whether a
similar one-step scaling law is observed.

\section{Conclusion}\label{sec:conclusion}

Motivated by confidence-based parallel unmasking, we analyzed a special
one-step Gaussian random-field model.  A locally dependent score field determines the \(K\) selected
indices, and a separately specified AR(1) kernel assigns a conditional Gaussian
total-correlation cost to their spacings.

Within this model, the cost vanishes under an explicit conservative budget and
correlation-decay regime.  At the square-root collision scale, adjacent selected
pairs have a Poisson limit and imply non-vanishing probability and expectation
lower bounds for the same Gaussian cost.  

The contribution is thus a stochastic-geometry baseline for one confidence-
ranked selection step. Its value is to separate a clean probabilistic mechanism--rare minima,
selected-set geometry, and a distance-dependent cost that can be tested when richer score and dependence structures are measured in trained
masked models.

\subsection*{Acknowledgement}
The author thanks Yifen Chen for helpful discussions.

\appendix

\section{Proofs for the first main theorem}
\label{app:first-main-proof}

This appendix proves the auxiliary sparsification theorem,
Theorem~\ref{thm:subcritical-spacing}, and then gives the formal deduction of
Theorem~\ref{thm:vanishing-tc}.  The argument passes from exact lowest-\(K\)
selection to a slightly denser threshold set, controls the threshold-set size
and short-range collisions, and finally combines the resulting separation with
the exact AR(1) total-correlation bound.

Throughout this appendix, \(\rhoV\) controls the variance field and
\(\rhoX\) enters only through the final total-correlation deduction.  The proof
has three components.  First, we establish a uniform two-point
small-ball upper bound whose excess over the independent value is summable over
distances.  Second, we control the cardinality and short-range collision count
of threshold sets.  Third, we enlarge the exact lowest-\(K\) rule to a slightly
denser threshold rule and transfer the estimates back to the exact order
statistic.

\subsection{Uniform two-point small-ball bounds}

The fixed-distance asymptotic in Lemma~\ref{lem:fixed-distance-small-ball} identifies the leading constant for each fixed distance.  For the subcritical theorem, however, we need a bound that is uniform in the distance and whose dependence on the distance is summable.  The next lemma supplies this estimate.

\begin{lemma}[Uniform two-point small-ball domination]\label{lem:uniform-two-point}
Fix \(m\ge1\) and \(0<|\rhoV|<1\).  There exist \(q_0\in(0,1)\) and a nonnegative sequence \((a_h)_{h\ge1}\) such that
\begin{equation}\label{eq:a-summable}
    A_{\rhoV} := \sum_{h=1}^{\infty} a_h < \infty,
\end{equation}
and, for every \(h\ge1\) and every \(q\in(0,q_0]\),
\begin{equation}\label{eq:uniform-two-point-bound}
    p_h(q)
    = \bP\bigl(V_1\le u_m(q),\,V_{1+h}\le u_m(q)\bigr)
    \le (1+a_h)q^2.
\end{equation}
Consequently there is a finite constant \(B_{\rhoV}\), depending only on \((\rhoV,m)\), such that
\begin{equation}\label{eq:crude-two-point-bound}
    p_h(q)\le B_{\rhoV} q^2,
    \qquad h\ge1,\quad q\in(0,q_0].
\end{equation}
\end{lemma}

\begin{proof}
Let \(G,H\in\R^m\) be centered Gaussian vectors satisfying
\[
    \Cov(G)=I_m,
    \qquad
    \Cov(H)=I_m,
    \qquad
    \Cov(G,H)=rI_m,
\]
where \(|r|<1\).  Relative to the product law of two independent \(N(0,I_m)\) vectors, the joint law of \((G,H)\) has Radon--Nikodym derivative
\begin{equation}\label{eq:rn-ratio}
    R_r(x,y)
    = (1-r^2)^{-m/2}
      \exp\left\{
      \frac{2r x\cdot y-r^2(\norm{x}^2+\norm{y}^2)}{2(1-r^2)}
      \right\}.
\end{equation}
Indeed, the covariance matrix is
\[
    \begin{pmatrix}
        I_m & rI_m \\
        rI_m & I_m
    \end{pmatrix},
\]
whose determinant is \((1-r^2)^m\), and its inverse is
\[
    \frac{1}{1-r^2}
    \begin{pmatrix}
        I_m & -rI_m \\
        -rI_m & I_m
    \end{pmatrix}.
\]
Dividing the corresponding joint density by the product standard Gaussian density gives \eqref{eq:rn-ratio}.

Set \(u_0=1\) and define
\begin{equation}\label{eq:q0-definition}
    q_0 := F_m(u_0).
\end{equation}
If \(0<q\le q_0\), then \(u_m(q)\le u_0\).  On the event
\[
    \norm{x}^2\le u_m(q),
    \qquad
    \norm{y}^2\le u_m(q),
\]
we have
\[
    2r x\cdot y-r^2(\norm{x}^2+\norm{y}^2)
    \le 2|r|\norm{x}\norm{y}
    \le 2|r|u_m(q)
    \le 2|r|u_0.
\]
Thus
\begin{equation}\label{eq:rn-ball-bound}
    R_r(x,y)
    \le
    (1-r^2)^{-m/2}
    \exp\left\{\frac{|r|u_0}{1-r^2}\right\}.
\end{equation}
For the AR(1) field, the correlation at distance \(h\) is \(r_h=\rhoV^h\).  Define
\begin{equation}\label{eq:ah-definition}
    1+a_h
    :=
    (1-r_h^2)^{-m/2}
    \exp\left\{\frac{|r_h|u_0}{1-r_h^2}\right\},
    \qquad h\ge1.
\end{equation}
The right-hand side is at least one, so \(a_h\ge0\).  Integrating \eqref{eq:rn-ball-bound} over the product event
\[
    \{\norm{x}^2\le u_m(q)\}\times \{\norm{y}^2\le u_m(q)\}
\]
under two independent standard Gaussian measures yields
\[
    p_h(q)
    \le (1+a_h)F_m(u_m(q))^2
    = (1+a_h)q^2,
\]
which proves \eqref{eq:uniform-two-point-bound}.

It remains to prove summability.  Since \(|r_h|=|\rhoV|^h\to0\), there exists \(h_0\) such that \(|r_h|\le1/2\) for all \(h\ge h_0\).  For such \(h\), using \(-\log(1-t)\le 2t\) for \(0\le t\le1/4\), we obtain
\[
    \log(1+a_h)
    = -\frac{m}{2}\log(1-r_h^2)
      + \frac{|r_h|u_0}{1-r_h^2}
    \le m r_h^2 + 2u_0|r_h|
    \le (m+2u_0)|r_h|.
\]
For all sufficiently large \(h\), the last bound is at most one.  Since \(e^x-1\le e x\) for \(0\le x\le1\), it follows that
\[
    a_h \le e(m+2u_0)|\rhoV|^h
\]
for all sufficiently large \(h\).  Hence \(\sum_h a_h<\infty\).  Finally, \eqref{eq:crude-two-point-bound} follows from \eqref{eq:uniform-two-point-bound} with
\[
    B_{\rhoV} := 1+\sup_{h\ge1}a_h < \infty.
\]
\end{proof}

\subsection{Threshold-set estimates}

For \(q\in(0,1)\), recall the threshold set
\[
    T_N(q)=\{i\in[N]:V_i\le u_m(q)\}
\]
and write
\begin{equation}\label{eq:MNq-definition}
    M_N(q):=|T_N(q)|.
\end{equation}
The next two lemmas are the only probabilistic inputs needed for exact lowest-\(K\) selection.

\begin{lemma}[Short-range collisions in a threshold set]\label{lem:threshold-collision}
There exists a finite constant \(C_{\mathrm{col}}=C_{\mathrm{col}}(\rhoV,m)\) such that, for all \(N\ge1\), all integers \(L\ge1\), and all \(q\in(0,q_0]\),
\begin{equation}\label{eq:threshold-collision-bound}
    \bP\bigl(C_L(T_N(q))\ge1\bigr)
    \le C_{\mathrm{col}} N L q^2.
\end{equation}
\end{lemma}

\begin{proof}
By Markov's inequality and the definition of \(C_L\),
\[
    \bP\bigl(C_L(T_N(q))\ge1\bigr)
    \le \bE C_L(T_N(q))
    = \sum_{h=1}^{L}\sum_{i=1}^{N-h}
      \bP(i\in T_N(q),\,i+h\in T_N(q)).
\]
If \(h\ge N\), the inner sum is empty; otherwise stationarity gives
\[
    \bP(i\in T_N(q),\,i+h\in T_N(q))=p_h(q).
\]
Using Lemma~\ref{lem:uniform-two-point},
\[
    \bE C_L(T_N(q))
    \le \sum_{h=1}^{L} (N-h)_+(1+a_h)q^2
    \le Nq^2\sum_{h=1}^{L}(1+a_h).
\]
Since \(L\ge1\) and \(\sum_{h\ge1}a_h=A_{\rhoV}<\infty\),
\[
    \sum_{h=1}^{L}(1+a_h)\le L+A_{\rhoV}\le (1+A_{\rhoV})L.
\]
Thus \eqref{eq:threshold-collision-bound} holds with \(C_{\mathrm{col}}=1+A_{\rhoV}\).
\end{proof}

\begin{lemma}[Concentration of the threshold-set size]\label{lem:threshold-concentration}
There exists a finite constant \(C_{\mathrm{var}}=C_{\mathrm{var}}(\rhoV,m)\) such that, for all \(N\ge1\) and all \(q\in(0,q_0]\),
\begin{equation}\label{eq:threshold-variance-bound}
    \Var(M_N(q))\le C_{\mathrm{var}}Nq.
\end{equation}
Consequently, for every deterministic sequence \(q_N\in(0,q_0]\) satisfying \(Nq_N\to\infty\),
\begin{equation}\label{eq:threshold-relative-concentration}
    \frac{M_N(q_N)}{Nq_N}\longrightarrow 1
    \qquad\text{in probability.}
\end{equation}
\end{lemma}

\begin{proof}
Let
\[
    I_i(q)=\mathbf{1}_{\{V_i\le u_m(q)\}},
    \qquad i\in[N].
\]
Then \(M_N(q)=\sum_{i=1}^N I_i(q)\), \(\bE I_i(q)=q\), and
\[
    \Var(M_N(q))
    =\sum_{i=1}^{N}\Var(I_i(q))
      +2\sum_{1\le i<j\le N}\Cov(I_i(q),I_j(q)).
\]
The diagonal part is bounded by
\[
    \sum_{i=1}^{N}\Var(I_i(q))\le Nq.
\]
For pairs at distance \(h\), Lemma~\ref{lem:uniform-two-point} gives
\[
    \bE[I_i(q)I_{i+h}(q)]\le (1+a_h)q^2,
\]
and hence
\[
    \Cov(I_i(q),I_{i+h}(q))
    =\bE[I_i(q)I_{i+h}(q)]-q^2
    \le a_hq^2.
\]
Therefore
\[
    2\sum_{1\le i<j\le N}\Cov(I_i(q),I_j(q))
    \le 2\sum_{h=1}^{N-1}(N-h)a_hq^2
    \le 2Nq^2A_{\rhoV}.
\]
Since \(q\le1\),
\[
    \Var(M_N(q))\le Nq+2Nq^2A_{\rhoV}
    \le (1+2A_{\rhoV})Nq.
\]
This proves \eqref{eq:threshold-variance-bound} with \(C_{\mathrm{var}}=1+2A_{\rhoV}\).

If \(Nq_N\to\infty\), then for every \(\varepsilon>0\), Chebyshev's inequality gives
\[
    \bP\left(\left|\frac{M_N(q_N)}{Nq_N}-1\right|>\varepsilon\right)
    \le
    \frac{\Var(M_N(q_N))}{\varepsilon^2N^2q_N^2}
    \le
    \frac{C_{\mathrm{var}}}{\varepsilon^2Nq_N}
    \longrightarrow0.
\]
Thus \eqref{eq:threshold-relative-concentration} follows.
\end{proof}

\subsection{From threshold sets to exact lowest-\texorpdfstring{$K$}{K} selection}

The following elementary comparison is the bridge from analytically convenient threshold sets to the exact lowest-\(K\) variance rule.

\begin{lemma}[Threshold enlargement contains the lowest-\(K\) set]\label{lem:threshold-enlargement}
Let \(1\le K\le N\), \(q\in(0,1)\), and \(u=u_m(q)\).  If \(M_N(q)\ge K\), then
\begin{equation}\label{eq:topk-subset-threshold}
    S_{N,K}\subseteq T_N(q).
\end{equation}
Consequently, for every integer \(L\ge1\),
\begin{equation}\label{eq:topk-threshold-collision-comparison}
    \{C_L(S_{N,K})\ge1\}
    \subseteq
    \{M_N(q)<K\}\cup\{C_L(T_N(q))\ge1\}.
\end{equation}
\end{lemma}

\begin{proof}
If \(M_N(q)\ge K\), then at least \(K\) of the variances \(V_1,\ldots,V_N\) are at most \(u\).  Hence the \(K\)-th order statistic of the multiset \(\{V_1,\ldots,V_N\}\) is at most \(u\).  Every index in \(S_{N,K}\) has score no larger than this \(K\)-th order statistic, and therefore has score at most \(u\).  This proves \eqref{eq:topk-subset-threshold}.  If \(C_L(S_{N,K})\ge1\) and \(M_N(q)\ge K\), then \(S_{N,K}\subseteq T_N(q)\), so the same pair is also counted by \(C_L(T_N(q))\).  This proves \eqref{eq:topk-threshold-collision-comparison}.
\end{proof}

\subsection{Proof of the subcritical spacing theorem}

We now prove the square-root sparsification theorem for the exact lowest-\(K\) set.

\begin{proof}[Proof of Theorem~\ref{thm:subcritical-spacing}]
Set
\begin{equation}\label{eq:alphaN-definition}
    \alpha_N := \frac{L_NK_N^2}{N}.
\end{equation}
By assumption, \(\alpha_N\to0\).  For all sufficiently large \(N\), define
\begin{equation}\label{eq:qN-choice}
    q_N := \frac{K_N}{N}\alpha_N^{-1/4}.
\end{equation}
Since \(\alpha_N>0\), this is well-defined.  We record the three elementary consequences of this choice.

First,
\begin{equation}\label{eq:qN-above-topk-density}
    \frac{q_N}{K_N/N}=\alpha_N^{-1/4}\longrightarrow\infty,
\end{equation}
so the threshold density is asymptotically much larger than the nominal selection fraction \(K_N/N\).  Equivalently,
\begin{equation}\label{eq:K-over-Nq}
    \frac{K_N}{Nq_N}=\alpha_N^{1/4}\longrightarrow0.
\end{equation}
Second,
\begin{equation}\label{eq:NLq2}
    NL_Nq_N^2
    =NL_N\left(\frac{K_N}{N}\right)^2\alpha_N^{-1/2}
    =\alpha_N^{1/2}
    \longrightarrow0.
\end{equation}
Third,
\begin{equation}\label{eq:Nq-diverges}
    Nq_N
    =K_N\alpha_N^{-1/4}
    =K_N^{1/2}\left(\frac{N}{L_N}\right)^{1/4}
    \longrightarrow\infty.
\end{equation}
To justify the last convergence, note that \(K_N\ge1\) and \(L_N/N=\alpha_N/K_N^2\le\alpha_N\to0\), so \((N/L_N)^{1/4}\to\infty\).  Also, \(q_N\to0\): using \(K_N/N=\sqrt{\alpha_N/(L_NN)}\),
\begin{equation}\label{eq:qN-to-zero}
    q_N=\frac{\alpha_N^{1/4}}{\sqrt{L_NN}}\le \alpha_N^{1/4}\longrightarrow0.
\end{equation}
Thus for all sufficiently large \(N\), \(q_N\in(0,q_0]\), where \(q_0\) is the constant from Lemma~\ref{lem:uniform-two-point}.

By Lemma~\ref{lem:threshold-concentration} and Chebyshev's inequality,
\[
    \bP(M_N(q_N)<K_N)
    \le
    \bP\left(|M_N(q_N)-Nq_N|>Nq_N-K_N\right).
\]
By \eqref{eq:K-over-Nq}, for all sufficiently large \(N\), \(Nq_N-K_N\ge \frac12Nq_N\).  Hence, using \eqref{eq:threshold-variance-bound},
\begin{equation}\label{eq:MN-less-K-bound}
    \bP(M_N(q_N)<K_N)
    \le
    \frac{4\Var(M_N(q_N))}{N^2q_N^2}
    \le
    \frac{4C_{\mathrm{var}}}{Nq_N}
    \longrightarrow0,
\end{equation}
where the last convergence follows from \eqref{eq:Nq-diverges}.

On the other hand, Lemma~\ref{lem:threshold-collision} and \eqref{eq:NLq2} give
\begin{equation}\label{eq:threshold-collision-goes-zero}
    \bP(C_{L_N}(T_N(q_N))\ge1)
    \le C_{\mathrm{col}}NL_Nq_N^2
    =C_{\mathrm{col}}\alpha_N^{1/2}
    \longrightarrow0.
\end{equation}
Combining \eqref{eq:topk-threshold-collision-comparison}, \eqref{eq:MN-less-K-bound}, and \eqref{eq:threshold-collision-goes-zero}, we obtain
\[
    \bP(C_{L_N}(S_{N,K_N})\ge1)
    \le
    \bP(M_N(q_N)<K_N)
    +\bP(C_{L_N}(T_N(q_N))\ge1)
    \longrightarrow0.
\]
This proves \eqref{eq:subcritical-collision-result}.  The equivalence with \eqref{eq:subcritical-spacing-result} follows from \eqref{eq:spacing-collision-equivalence}.  Finally, taking \(L_N\equiv1\) turns \eqref{eq:subcritical-assumption} into \(K_N^2/N\to0\), which is exactly \(K_N=o(\sqrt N)\), and \(C_{1}(S)=A(S)\) by \eqref{eq:adjacent-count}.  Thus \eqref{eq:no-adjacent-result} follows.
\end{proof}

\subsection{Deduction of the first main theorem from sparsification}

We next prove the total-correlation theorem.  The deterministic input is Lemma~\ref{lem:tc-bound}: once the selected sites are separated, the retained AR(1) correlations are uniformly small.

\begin{proof}[Proof of Theorem~\ref{thm:vanishing-tc}]
If \(K_N=1\), then \(\TC_{\Sigma_N}(S_{N,K_N})=0\).  Otherwise define
\[
    E_N:=\{\Delta(S_{N,K_N})>L_N\}.
\]
Assumption~\eqref{eq:subcritical-assumption-main} is the assumption of
Theorem~\ref{thm:subcritical-spacing}, so \(\bP(E_N)\to1\).  On \(E_N\),
Lemma~\ref{lem:tc-bound} gives
\begin{equation}\label{eq:tc-proof-exact-bound}
    0\le \TC_{\Sigma_N}(S_{N,K_N})
    \le
    \frac{K_N-1}{2}\,
    \frac{|\rhoX|^{2(L_N+1)}}{1-|\rhoX|^{2(L_N+1)}}.
\end{equation}
Because \(L_N\ge1\) and \(0<|\rhoX|<1\), the denominator is bounded below by
\(1-|\rhoX|^4>0\).  The deterministic right-hand side therefore tends to zero
under \eqref{eq:decorrelation-assumption-section3}.  For every \(\varepsilon>0\),
for all sufficiently large \(N\),
\[
    \bP\bigl(\TC_{\Sigma_N}(S_{N,K_N})>\varepsilon\bigr)
    \le \bP(E_N^c)\longrightarrow0.
\]
This proves \eqref{eq:tc-convergence-probability}.
\end{proof}

\section{Proof of the auxiliary critical adjacent-pair limit}
\label{app:critical-poisson-proof}

This appendix proves Theorem~\ref{thm:exact-critical-poisson}.  The proofs for
Theorem~\ref{thm:vanishing-tc} and its auxiliary sparsification theorem are in
Appendix~\ref{app:first-main-proof}.  The present proof first
establishes uniform finite-set small-ball bounds, then asymptotic decoupling for
well-separated adjacent blocks, and finally a factorial-moment Poisson limit for
threshold selection.  A two-sided threshold sandwich transfers that limit to the
exact lowest-\(K\) set.

For a threshold density \(q\in(0,1)\), define the adjacent threshold event
\begin{equation}\label{eq:Bi-definition}
    B_i(q)=\{V_i\le u_m(q),\ V_{i+1}\le u_m(q)\},
    \qquad 1\le i\le N-1,
\end{equation}
and its count
\begin{equation}\label{eq:threshold-adjacent-count}
    A_N(q)=\sum_{i=1}^{N-1}\mathbf 1_{B_i(q)}.
\end{equation}

\subsection{Auxiliary small-ball estimates}

We first record two estimates for finite collections of small-ball events.  The
first is a uniform upper bound for any fixed number of sites.  The second says
that well-separated adjacent blocks asymptotically decouple.

\begin{lemma}[Uniform finite-set small-ball upper bound]\label{lem:finite-set-small-ball}
For every integer \(s\ge1\), there exist constants \(q_s\in(0,1)\) and
\(C_s<\infty\), depending only on \((s,\rhoV,m)\), such that for every \(N\), every
set of distinct indices \(i_1,\ldots,i_s\in[N]\), and every \(q\in(0,q_s]\),
\begin{equation}\label{eq:finite-set-small-ball}
    \bP\bigl(V_{i_1}\le u_m(q),\ldots,V_{i_s}\le u_m(q)\bigr)
    \le C_s q^s.
\end{equation}
\end{lemma}

\begin{proof}
Order the indices as \(i_1<\cdots<i_s\), and let
\(R=(\rhoV^{\abs{i_a-i_b}})_{a,b=1}^s\).  The Gaussian vector
\((Y_{i_1},\ldots,Y_{i_s})\in(\R^m)^s\) has covariance \(R\otimes I_m\).
We compare its law with the product law of \(s\) independent \(N(0,I_m)\)
vectors.

We first note that the eigenvalues of \(R\) are bounded away from zero by a
constant depending only on \((s,\rhoV)\).  Indeed, if
\(g_a=i_{a+1}-i_a\), then the scalar AR(1) Markov property gives
\begin{equation}\label{eq:submatrix-determinant}
    \det R=
    \prod_{a=1}^{s-1}(1-\rhoV^{2g_a})
    \ge (1-\rhoV^2)^{s-1}.
\end{equation}
Also \(\lambda_{\max}(R)\le s\), since every entry of \(R\) has absolute value at
most one.  Hence
\begin{equation}\label{eq:R-eigen-lower}
    \lambda_{\min}(R)
    \ge \frac{(1-\rhoV^2)^{s-1}}{s^{s-1}}
    =:c_{s,\rhoV}>0.
\end{equation}
Consequently both \(\det(R)^{-m/2}\) and \(\norm{R^{-1}-I_s}_{\mathrm{op}}\)
are bounded by constants depending only on \((s,\rhoV,m)\).

The density ratio of the joint law with respect to the product standard Gaussian
law is
\begin{equation}\label{eq:finite-set-density-ratio}
    \mathcal R_R(x_1,
    \ldots,x_s)
    =\det(R)^{-m/2}
    \exp\left\{-\frac12
    \sum_{a,b=1}^{s}(R^{-1}-I_s)_{ab}x_a\cdot x_b\right\}.
\end{equation}
Fix \(u_0>0\) and set \(q_s:=F_m(u_0)\).  On the event
\(\norm{x_a}^2\le u_m(q)\le u_0\) for all \(a\), the exponent in
\eqref{eq:finite-set-density-ratio} is bounded above by a constant depending
only on \((s,\rhoV,m,u_0)\).  Thus \(\mathcal R_R\le C_s\) on this event, uniformly
over the choice of the indices.  Integrating over the product event gives
\[
    \bP\bigl(\norm{Y_{i_a}}^2\le u_m(q),\ 1\le a\le s\bigr)
    \le C_s\prod_{a=1}^s\bP(\chi_m^2\le u_m(q))
    =C_s q^s.
\]
This proves the lemma.
\end{proof}

\begin{lemma}[Asymptotic decoupling of separated adjacent blocks]\label{lem:separated-blocks}
Fix an integer \(r\ge1\).  Let \(q_N\downarrow0\) and let \(g_N\to\infty\) be an
integer sequence.  Let
\begin{equation}\label{eq:pN-section4}
    p_N:=p_1(q_N)=\bP(B_1(q_N)).
\end{equation}
Uniformly over all \(r\)-tuples of distinct indices
\(i_1,
\ldots,i_r\in\{1,\ldots,N-1\}\) satisfying
\begin{equation}\label{eq:separated-block-condition}
    \abs{i_a-i_b}>g_N+1,
    \qquad a\ne b,
\end{equation}
one has
\begin{equation}\label{eq:separated-block-factorization}
    \bP\left(\bigcap_{a=1}^{r}B_{i_a}(q_N)\right)
    =p_N^r(1+o(1)),
\end{equation}
where the \(o(1)\) term may depend on \((r,\rhoV,m,q_N,g_N)\) but is uniform in
the locations of the blocks.
\end{lemma}

\begin{proof}
The case \(r=1\) is immediate, so assume \(r\ge2\).  For each block define
\[
    W_a=(Y_{i_a},Y_{i_a+1})\in\R^{2m},
    \qquad 1\le a\le r.
\]
Each \(W_a\) has covariance
\[
    \Gamma=
    \begin{pmatrix}1&\rhoV\\ \rhoV&1\end{pmatrix}\otimes I_m.
\]
Let \(\widetilde W_a=\Gamma^{-1/2}W_a\).  Under the product law of independent
blocks, \((\widetilde W_1,
\ldots,
\widetilde W_r)\) is standard Gaussian in
\(\R^{2mr}\).  Under the true joint law, its covariance matrix has the form
\[
    I_{2mr}+E_N,
\]
with zero diagonal \(2m\times2m\) blocks.  If \eqref{eq:separated-block-condition}
holds, then every site in block \(a\) is at distance at least \(g_N+1\) from every
site in block \(b\), for \(a\ne b\).  Therefore the unwhitened cross-covariance
between the two blocks has operator norm at most \(C_{\rhoV,m}\abs{\rhoV}^{g_N}\).
Since \(\Gamma^{-1/2}\) is fixed, there is a constant \(C_{r,\rhoV,m}\) such that
\begin{equation}\label{eq:EN-op-bound}
    \norm{E_N}_{\mathrm{op}}
    \le C_{r,\rhoV,m}\abs{\rhoV}^{g_N}
    =: \varepsilon_N
    \longrightarrow0.
\end{equation}
For all sufficiently large \(N\), \(\varepsilon_N<1/2\).

The density ratio of the true joint law of the whitened vector
\(z=(z_1,
\ldots,z_r)\in\R^{2mr}\) with respect to the product standard Gaussian law is
\begin{equation}\label{eq:block-density-ratio}
    \mathcal R_N(z)
    =\det(I_{2mr}+E_N)^{-1/2}
      \exp\left\{-\frac12 z^\top\bigl((I_{2mr}+E_N)^{-1}-I_{2mr}\bigr)z\right\}.
\end{equation}
By \eqref{eq:EN-op-bound},
\[
    \abs{\log\det(I_{2mr}+E_N)}\le C_{r,\rhoV,m}\varepsilon_N
\]
and
\[
    \norm{(I_{2mr}+E_N)^{-1}-I_{2mr}}_{\mathrm{op}}
    \le C_{r,\rhoV,m}\varepsilon_N.
\]
On the event \(\bigcap_a B_{i_a}(q_N)\), each unwhitened block satisfies
\(\norm{W_a}^2\le2u_m(q_N)\).  Since \(\Gamma^{-1/2}\) is fixed,
\begin{equation}\label{eq:z-bound-block-event}
    \norm{z}^2\le C_{r,\rhoV,m}u_m(q_N).
\end{equation}
As \(q_N\downarrow0\), \(u_m(q_N)\downarrow0\).  Combining
\eqref{eq:block-density-ratio}--\eqref{eq:z-bound-block-event}, we get
\[
    \sup_{z\in\cap_a B_{i_a}(q_N)}\abs{\log\mathcal R_N(z)}=o(1),
\]
uniformly over all block locations satisfying \eqref{eq:separated-block-condition}.
Thus \(\mathcal R_N=1+o(1)\) uniformly on the block event.  Integrating with
respect to the product block law gives
\[
    \bP\left(\bigcap_{a=1}^{r}B_{i_a}(q_N)\right)
    =(1+o(1))\prod_{a=1}^{r}\bP(B_{i_a}(q_N))
    =p_N^r(1+o(1)),
\]
which is \eqref{eq:separated-block-factorization}.
\end{proof}

\subsection{Poisson limit for threshold adjacent pairs}

We next prove the critical Poisson limit for the threshold set.  This is the
probabilistic core of the exact lowest-\(K\) result.

\begin{proposition}[Threshold adjacent-pair Poisson limit]\label{prop:threshold-poisson}
Let \((q_N)\) satisfy
\begin{equation}\label{eq:q-critical-threshold}
    \sqrt N\,q_N\longrightarrow a
    \qquad\text{for some }a\in(0,\infty).
\end{equation}
Then
\begin{equation}\label{eq:threshold-poisson-limit}
    A_N(q_N)
    \Rightarrow
    \operatorname{Poisson}\bigl(\vartheta_{\rhoV,m}a^2\bigr).
\end{equation}
\end{proposition}

\begin{proof}
Set \(p_N=p_1(q_N)\).  By Lemma~\ref{lem:fixed-distance-small-ball},
\begin{equation}\label{eq:pN-critical-asymptotic}
    p_N=\vartheta_{\rhoV,m}q_N^2(1+o(1)).
\end{equation}
Hence the mean
\begin{equation}\label{eq:threshold-mean-critical}
    \mu_N:=\bE A_N(q_N)=(N-1)p_N
\end{equation}
satisfies
\begin{equation}\label{eq:muN-critical-limit}
    \mu_N\longrightarrow \mu:=\vartheta_{\rhoV,m}a^2.
\end{equation}

We prove convergence of all factorial moments.  Let \((x)_r=x(x-1)\cdots(x-r+1)\).
For fixed \(r\ge1\),
\begin{equation}\label{eq:factorial-moment-sum}
    \bE\bigl[(A_N(q_N))_r\bigr]
    =\sum_{\mathbf i\in\mathcal I_{N,r}}
      \bP\left(\bigcap_{a=1}^{r}B_{i_a}(q_N)\right),
\end{equation}
where \(\mathcal I_{N,r}\) is the set of ordered \(r\)-tuples of distinct indices
from \(\{1,\ldots,N-1\}\).

Choose
\begin{equation}\label{eq:gN-critical-choice}
    g_N=\lceil\log N\rceil.
\end{equation}
Call \(\mathbf i=(i_1,
\ldots,i_r)\in\mathcal I_{N,r}\) good if
\(\abs{i_a-i_b}>g_N+1\) for all \(a\ne b\), and bad otherwise.  The number of bad
ordered tuples is \(O_r(N^{r-1}g_N)\), so the number of good tuples is
\begin{equation}\label{eq:good-tuple-count}
    N^r(1+o(1)).
\end{equation}
By Lemma~\ref{lem:separated-blocks}, the contribution of good tuples to
\eqref{eq:factorial-moment-sum} is
\begin{equation}\label{eq:good-tuple-contribution}
    N^r p_N^r(1+o(1))
    = (N p_N)^r(1+o(1))
    \longrightarrow \mu^r.
\end{equation}

It remains to show that bad tuples contribute \(o(1)\).  Sort the starting points
of a bad tuple as \(j_1<\cdots<j_r\), and partition them into clusters by placing
a break between \(j_\ell\) and \(j_{\ell+1}\) exactly when
\(j_{\ell+1}-j_\ell>g_N+1\).  If there are \(c\) clusters, then a bad tuple has
\(1\le c\le r-1\).  For a fixed value of \(c\), the number of possible sorted
clustered configurations is at most
\begin{equation}\label{eq:cluster-count}
    C_r N^c(g_N+1)^{r-c},
\end{equation}
and passing from sorted to ordered tuples only changes the constant \(C_r\).

A cluster containing \(e\) distinct adjacent-pair starts involves at least
\(e+1\) distinct sites of the underlying chain.  Since different clusters are
separated, a configuration with \(r\) edge starts and \(c\) clusters involves at
least \(r+c\) distinct sites.  The event attached to the tuple forces all these distinct sites to be below the
same threshold.  Since \(q_N\downarrow0\), Lemma~\ref{lem:finite-set-small-ball}
may be applied uniformly over all possible values of \(s\le2r\).  The probability
attached to any such tuple is therefore bounded by
\begin{equation}\label{eq:bad-tuple-prob-bound}
    C_r q_N^{r+c},
\end{equation}
for all sufficiently large \(N\), after increasing \(C_r\) if necessary.  The total bad contribution is thus at most
\begin{equation}\label{eq:bad-contribution-bound}
    C_r\sum_{c=1}^{r-1}N^c(g_N+1)^{r-c}q_N^{r+c}.
\end{equation}
Since \(q_N=O(N^{-1/2})\) and \(g_N=O(\log N)\), each summand in
\eqref{eq:bad-contribution-bound} is bounded by
\[
    C_r(\log N)^{r-c}N^{c-(r+c)/2}
    =C_r(\log N)^{r-c}N^{-(r-c)/2},
\]
which tends to zero because \(c\le r-1\).  Hence the total bad contribution is
\(o(1)\).

Combining the good and bad contributions gives
\[
    \bE\bigl[(A_N(q_N))_r\bigr]\longrightarrow \mu^r
    \qquad\text{for every fixed }r\ge1.
\]
The convergence of factorial moments implies convergence of the corresponding
ordinary moments, because ordinary moments are finite linear combinations of
factorial moments of lower order.  The bounded first moments give tightness.  For
each fixed integer \(r\), bounded \((r+1)\)-st moments imply uniform integrability
of the \(r\)-th powers, so every subsequential weak limit has the Poisson moments
\(\mu^r\) in factorial form.  The Poisson distribution is moment-determinate
because its moment-generating function is finite in a neighborhood of the
origin.  The method of moments therefore yields
\eqref{eq:threshold-poisson-limit}.
\end{proof}

\subsection{The exact lowest-\texorpdfstring{$K$}{K} critical window}

We now transfer the threshold Poisson limit to the exact order statistic and prove Theorem~\ref{thm:exact-critical-poisson}.  The argument is a two-sided threshold sandwich.  It is important that the threshold set has size of order \(\sqrt N\) but fluctuations only of order \(N^{1/4}\), so fixed multiplicative perturbations of the threshold density contain the random lowest-\(K\) cutoff with high probability.

\begin{proof}[Proof of Theorem~\ref{thm:exact-critical-poisson}]
Fix \(\varepsilon\in(0,1)\), and define
\begin{equation}\label{eq:qpm-definition}
    q_N^-=(1-\varepsilon)\frac{K_N}{N},
    \qquad
    q_N^+=(1+\varepsilon)\frac{K_N}{N}.
\end{equation}
For all large \(N\), both numbers lie in \((0,1)\), and
\[
    \sqrt N q_N^\pm\longrightarrow (1\pm\varepsilon)\lambda.
\]
Let
\[
    M_N^\pm=|T_N(q_N^\pm)|,
    \qquad
    A_N^\pm=A_N(q_N^\pm).
\]
By Lemma~\ref{lem:threshold-concentration}, or directly by the variance bound
\eqref{eq:threshold-variance-bound},
\begin{align}
    \bP(M_N^->K_N)
    &\le
    \frac{\Var(M_N^-)}{(K_N-Nq_N^-)^2}
    \le
    \frac{C_{\mathrm{var}}Nq_N^-}{\varepsilon^2K_N^2}
    \longrightarrow0,\label{eq:lower-size-sandwich}\\[4pt]
    \bP(M_N^+<K_N)
    &\le
    \frac{\Var(M_N^+)}{(Nq_N^+-K_N)^2}
    \le
    \frac{C_{\mathrm{var}}Nq_N^+}{\varepsilon^2K_N^2}
    \longrightarrow0.\label{eq:upper-size-sandwich}
\end{align}
On the event
\begin{equation}\label{eq:sandwich-event}
    E_N(\varepsilon):=\{M_N^-\le K_N\le M_N^+\},
\end{equation}
the order statistic satisfies
\begin{equation}\label{eq:set-sandwich}
    T_N(q_N^-)
    \subseteq S_{N,K_N}
    \subseteq T_N(q_N^+).
\end{equation}
Indeed, if \(M_N^-\le K_N\), every site below the lower threshold is among the
\(K_N\) smallest scores; if \(M_N^+\ge K_N\), the \(K_N\)-th smallest score is at
most the upper threshold.  Hence, on \(E_N(\varepsilon)\),
\begin{equation}\label{eq:adjacent-sandwich}
    A_N^-\le A(S_{N,K_N})\le A_N^+.
\end{equation}
By \eqref{eq:lower-size-sandwich} and \eqref{eq:upper-size-sandwich},
\(\bP(E_N(\varepsilon))\to1\).

Proposition~\ref{prop:threshold-poisson} gives
\begin{equation}\label{eq:Apm-poisson}
    A_N^\pm
    \Rightarrow
    \operatorname{Poisson}\bigl(\vartheta_{\rhoV,m}(1\pm\varepsilon)^2\lambda^2\bigr).
\end{equation}
Let \(F_\gamma(\ell)=\bP(\operatorname{Poisson}(\gamma)\le\ell)\).  For every
integer \(\ell\ge0\), the event \(E_N(\varepsilon)\) and the sandwich
\eqref{eq:adjacent-sandwich} give
\[
    \bP(A_N^+\le \ell)-\bP(E_N(\varepsilon)^c)
    \le
    \bP(A(S_{N,K_N})\le \ell)
    \le
    \bP(A_N^-\le \ell)+\bP(E_N(\varepsilon)^c).
\]
Using \(\bP(E_N(\varepsilon)^c)\to0\) and \eqref{eq:Apm-poisson}, we obtain
\begin{align}
    F_{\vartheta_{\rhoV,m}(1+\varepsilon)^2\lambda^2}(\ell)
    &\le
    \liminf_{N\to\infty}\bP(A(S_{N,K_N})\le\ell)
    \notag\\
    &\le
    \limsup_{N\to\infty}\bP(A(S_{N,K_N})\le\ell)
    \le
    F_{\vartheta_{\rhoV,m}(1-\varepsilon)^2\lambda^2}(\ell).
    \label{eq:cdf-sandwich}
\end{align}
Finally let \(\varepsilon\downarrow0\).  The Poisson distribution function is
continuous in its mean parameter, so the lower and upper bounds in
\eqref{eq:cdf-sandwich} both converge to
\(F_{\vartheta_{\rhoV,m}\lambda^2}(\ell)\).  Therefore the distribution functions
of \(A(S_{N,K_N})\) converge at every integer \(\ell\).  Since all variables
are supported on \(\mathbb N_0\), differences of consecutive cdf values give
convergence of every point mass, and hence weak convergence.  This proves
\eqref{eq:exact-critical-poisson}.  Taking \(\ell=0\) gives
\eqref{eq:exact-critical-zero-prob}.
\end{proof}

\phantomsection
\end{document}